\documentclass{article}
\usepackage{cite}

\usepackage{amsmath,amsthm,amstext,amsfonts,amssymb,mathrsfs}
\usepackage{algorithm, algorithmic}
\usepackage{graphicx}
\usepackage{textcomp}
\usepackage{xcolor}
\usepackage{bbm}
\usepackage{caption}
\usepackage{picture,xcolor}

\usepackage{tgbonum}

\newcommand{\Pp}{\mathbb{P}}
\newcommand{\E}{\mathbb{E}}

\newtheorem{theorem}{Theorem}

\newtheorem{lemma}{Lemma}
\newtheorem{proposition}{Proposition}
\newtheorem{assumption}{Assumption}

\theoremstyle{remark} 
\newtheorem{remark}{Remark} 

\usepackage[left=1in,right=1in,top=1in,bottom=1in]{geometry}
\date{}

\title{
Bounded Regret for Finitely Parameterized Multi-Armed Bandits
}

\author{Kishan Panaganti and Dileep Kalathil% <-this % stops a space
	\thanks{Authors are with the Department of Electrical and Computer Engineering at Texas A\&M University, College Station, TX, USA. Email:{\tt \{kpb, dileep.kalathil\}@tamu.edu }}
	\thanks{This work was supported in part by the National Science Foundation  Grant CRII:CPS-1850206}
}

\begin{document}

\maketitle
\thispagestyle{empty}
\pagestyle{empty}

\begin{abstract}
	We consider the problem of finitely parameterized multi-armed bandits where the model of the underlying stochastic environment can be characterized based on a common unknown parameter. The true parameter is unknown to the learning agent. However, the set of possible parameters, which is finite,  is known a priori. We propose an algorithm that is simple and easy to implement, which we call  Finitely Parameterized Upper Confidence Bound (FP-UCB) algorithm, which uses the information about the underlying parameter set for faster learning. In particular, we show that the FP-UCB algorithm achieves a bounded regret under some structural condition on the underlying parameter set.  We also show that, if the underlying parameter set does not satisfy the necessary structural condition, the FP-UCB algorithm achieves a logarithmic regret, but with a smaller preceding constant compared to the standard UCB algorithm. We also validate the superior performance of the FP-UCB algorithm through extensive numerical simulations. 
\end{abstract}

%\begin{IEEEkeywords}
%Multi-Armed Bandits, Reinforcement Learning Learning, Sequential Decision Making
%\end{IEEEkeywords}

%\begin{keywords}
%	Multi-Armed Bandits, Reinforcement Learning Learning, Sequential Decision Making
%\end{keywords}

\section{Introduction}

Multi-Armed Bandits (MAB) problems are canonical formalism for studying how an  agent learns to take optimal actions by repeated interactions with a stochastic  environment. The learning agent receives a reward at each time step which will depend on the action of the agent as well as the stochastic uncertainty associated with the environment. The goal of the  agent is to take actions in such a way  to maximize the cumulative reward. When the  model of the environment is perfectly known, computing the optimal action is  often a straightforward optimization problem. The challenge, as in the case of most real-world problems, is that agent does not know the stochastic model of environment a priori. The  agent needs to do  \textit{exploration}, i.e., take various actions sequentially to gather information,  in order to estimate the   model of the system. At the same time, the  agent needs to do \textit{exploitation} of the available information at any given time for maximizing the cumulative  reward. This  \textit{exploration vs. exploitation} trade-off is at the core of the MAB problems.   

 %Multi-armed bandits problems have been studied extensively in the literature. 
 
 Lai and Robbins in their seminal paper \cite{lai1985asymptotically}  formulated the non-Bayesian stochastic MAB problem and characterized the performance of a learning algorithm using the metric of \textit{regret}. They showed that no learning algorithm will be able to achieve a regret better than $O(\log T)$. They also proposed a learning algorithm that achieves an asymptotic logarithmic regret, matching the fundamental lower bound. A simple index-based algorithm called UCB algorithm was introduced in \cite{auer2002finite} which achieves the order optimal regret  in a non-asymptotic manner. This approach led to the development of a number of interesting algorithms, like   linear bandits \cite{dani2008stochastic}, contextual bandits \cite{chu2011contextual}, combinatorial bandits \cite{cesa2012combinatorial}, and decentralized and multi-player bandits \cite{kalathil2014decentralized}. 
 
 %Ananthram {et al.} later extended this to  the more general setting of Markovian rewards and multiple plays \cite{anantharam1987asymptotically-a, anantharam1987asymptotically-b}. 
 
Thompson (Posterior) Sampling  is another class of algorithms that gives superior numerical performance for MAB problems. Posterior sampling heuristic was first introduced by Thompson \cite{thompson1933likelihood}, but the first rigorous performance guarantee, an $O(\log T)$ regret, was given in \cite{agrawal2012analysis}. Thompson sampling idea has been used to develop algorithms for  bandits with multiple plays \cite{komiyama2015optimal}, contextual bandits \cite{agrawal2013thompson}, general online learning problem \cite{gopalan2014thompson}, and reinforcement learning \cite{osband2013more}. Both classes of algorithms have been used in a number of practical applications, like communication networks \cite{tekin2012approximately}, smart grids \cite{kalathil2015online}, and recommendation systems \cite{zong2016cascading}.

{\bf Our contribution:} We consider a class of multi-armed bandits problems where the reward corresponding to each arm can be characterized based on a common  unknown parameter.  In particular, we consider the setting where the cardinality of the set of possible parameters is finite.  This is inspired  by many real-world applications. For example, in recommendation systems and e-commerce applications (Amazon, Netflix), it is typical to assume that each user has a certain `type' parameter (denoted as $\theta$ in our problem formulation), and the set of possible parameters is finite. The preferences of the user is characterized by her type (for example, prefer science books over fiction books). The set of all possible types and the preferences of each type may be known a priori, but the  type of a new user may be unknown. So, instead of learning the preferences of this user over all possible choices, it may be easier to learn the type parameter of this user from a few observations. In this work, we propose an algorithm that explicitly uses the availability of such structural information about the underlying parameter set which enables a faster learning.

 We propose an algorithm that is simple and easy to implement, which we call FP-UCB algorithm, that uses the structural information  for faster learning. We show that the proposed FP-UCB algorithm can achieve a bounded regret $(O(1))$ under some structural condition on the underlying parameter set. This is in sharp contrast to the increasing $(O(\log T))$ regret of the standard multi-armed bandits algorithms. We also show that, if the underlying parameter set does not satisfy the necessary structural condition, the FP-UCB algorithm achieves a regret of $O(\log T)$, but with a smaller preceding constant compared to the standard UCB algorithm. The regret achieved by our algorithm also matches with the fundamental lower bound given by \cite{rajeev1989asymptotically}.  One remarkable aspect of our algorithm is that, it is oblivious to the fact  if the underlying parameter set satisfies the necessary condition or not, and thus avoiding re-tuning of the algorithm depending on the problem instance. Instead, it achieves the best possible performance given the problem instance. 

{\bf Related work:} Finitely parameterized multi-armed bandits problem was first studied by Agrawal et al. \cite{rajeev1989asymptotically}. They also proposed an algorithm for this setting, and proved that their algorithm achieves a bounded regret when the parameter set satisfies some necessary condition, and logarithmic regret otherwise. However, their algorithm is rather complicated which limits practical implementations and extension to other settings. The regret analysis is also involved and asymptotic in nature,  different from the  recent simpler index-based bandits algorithms and  their finite time analysis. \cite{rajeev1989asymptotically} also provided a fundamental lower bound for this class of problems. Compared to this work, our  FP-UCB algorithm is  simple, easy to implement, and easy to analyze, while providing non-asymptotic performance guarantees, which matches the lower bound. 

There are some recent works on exploiting the available structure of the MAB problem for getting  tighter regret bounds. In particular,   \cite{atan2015global} \cite{lattimore2014bounded}   \cite{maillard2014latent} \cite{combes2017minimal}  consider the problem setting similar to our paper where the mean reward of each arm is parameterized by a single unknown  parameter.  \cite{atan2015global} assumes that the reward functions are continuous in the global parameter and gives a bounded regret result.  \cite{lattimore2014bounded} gives specific conditions on the mean reward to achieve a bounded regret. \cite{maillard2014latent} considers a latent bandit problem where the reward distributions are  partitioned into a  number of clusters and indexed by a latent  parameter corresponding to the cluster.  \cite{combes2017minimal}  characterizes the minimal rates at which sub-optimal arms have to be explored depending on the structural information, and proposes an algorithm which achieves these rates.   \cite{bubeck2013bounded}   \cite{bubeck2013prior}   \cite{vakili2013achieving} exploit a different structural information where   it is shown that if  the mean value of the best arm and the second best arm (but not the identity of the arms) are known, then a bounded regret can be achieved. There are also  works on bandits algorithms that try to exploit the side information \cite{wang2005bandit} \cite{caron2012leveraging}, and recently in the context of  contextual bandits \cite{bastani2017mostly}.  Our problem formulation, algorithm, and analysis are very different from these works.  We also note that our problem formulation is fundamentally different from the system identification problems \cite{ljung1998system} \cite{kumar2015stochastic} because the goal  here is to learn an optimal policy online.

\section{Problem Formulation}

We consider the following sequential decision making problem. In each time step  $t \in \{1, 2, \ldots, T\},$ the  agent   selects an arm (action) from the set of $L$ possible arms, denoted as, $a(t) \in [L] = \{1, \ldots, L\}$.  Each arm $i$, when selected, yields a random real-valued reward. More precisely, let $X_{i}(\tau)$  be the random reward from arm $i$  in its $\tau$th selection. We assume that  $X_{i}(\tau)$  is drawn according to a probability distribution $P_{i}(\cdot ; \theta^{o})$ with a mean $\mu_{i}(\theta^{o})$. Here $\theta^{o}$ is the (true) parameter that determines the distribution of the stochastic rewards. The agent does not know $\theta^{o}$ or the corresponding mean values $\mu_{i}(\theta^{o})$. The random reward obtained from playing an arm repeatedly are i.i.d. and independent of the plays of the other arms.  We assume that  rewards are bounded with support in $[0,1]$.  The goal of the agent is to select a sequence of  actions  that maximizes the expected cumulative reward, $\mathbb{E}[\sum^{T}_{t=1} \mu_{a(t)}(\theta^{o}))]$.  The action $a(t)$  depends on the history of observations available to the agent until time $t$. So, $a(t)$ is stochastic and the expectation is with respect to all the  possible randomness. 

Clearly, the optimal choice is to select the  best arm (the arm with the highest mean value)  all the time, i.e., $a(t) = a^{*}(\theta^{o}), \forall t$, where $a^{*}(\theta^{o}) = \arg \max_{i \in [L]} \mu_{i}(\theta^{o})$.  However, the agent will be able to make this optimal decision only if  she knows the parameter $\theta^{o}$ or the corresponding mean values  $\mu_{i}(\theta^{o})$ for all $i$. The goal of a MAB  algorithm is to learn to make the optimal sequence of  decisions  without knowing the true  parameter $\theta^{o}$.

We consider the setting where the agent knows the set of possible  parameters $\Theta$. We assume that $\Theta$ is finite.  If the true parameter were $\theta \in \Theta$, then  agent selecting arm $i$ will get a random reward  drawn according to a distribution $P_{i}(\cdot ; \theta)$ with a mean $\mu_{i}(\theta)$.  We assume that for each $\theta \in \Theta$, the agent knows $P_{i}(\cdot ; \theta)$ and $\mu_{i}(\theta)$ for all $i \in [L]$. The optimal arm corresponding to the parameter $\theta$ is denoted as  $a^{*}(\theta) = \arg \max_{i \in [L]} \mu_{i}(\theta)$. We emphasize that the agent does not know the true parameter $\theta^{o}$ (and hence the optimal action $a^{*}(\theta^{o})$) except the fact that  it  is in the finite set $\Theta$.

In the multi-armed bandits literature, it is standard to characterize the performance of an online learning algorithm using the metric of regret. Regret is defined as the performance loss of an algorithm as compared to the optimal algorithm with complete information. Since $b(t) = a^{*}(\theta^{o})$, the expected cumulative regret of a multi-armed bandits algorithm after $T$ time steps is defined as
\begin{align}
\label{eq:regret-defn1}
\mathbb{E}[R(T)] := \mathbb{E}\left[\sum^{T}_{t=1} (\mu_{a^{*}(\theta^{o})}(\theta^{o}) -\mu_{a(t)}(\theta^{o}))\right].
\end{align}

The goal of a multi-armed bandits learning algorithm is to select actions sequentially in order to minimize $\mathbb{E}[R(T)]$.

\section{UCB Algorithm for Finitely Parameterized Multi-Armed Bandits}

In this section, we present our algorithm for finitely parameterized multi-armed bandits and the main theorem. We first introduce  a few notations for presenting the algorithm and the results succinctly.

Let $n_{i}(t)$ be the number of times arm $i$ has been selected by the algorithm until time $t$, i.e., $n_{i}(t) = \sum^{t}_{\tau = 1} \mathbbm{1}\{a(\tau) = i\}$. Here $\mathbbm{1}\{.\}$ is an indicator function. Define the empirical mean corresponding to arm $i$ at time $t$ as,
\begin{align}
\label{eq:muhat}
\hat{\mu}_{i}(t) := \frac{1}{n_i(t)} \sum^{n_i(t)}_{\tau = 1} X_{i}(\tau).
\end{align}

Define the set $A:=\{ a^{*}(\theta) : \theta\in\Theta \}$, which is the collection of  optimal arms corresponding to  all  parameters in $\Theta$.  Intuitively, a learning agent can restrict to selecting the arms from the set $A$. Clearly, $A \subset [L]$ and this reduction can be  useful when $|A|$ is much smaller than $L$.

Our FP-UCB Algorithm is given in Algorithm \ref{alg:fp-ucb}.  Figure \ref{fig:episodes} gives an illustration of the episodes and time slots of the FP-UCB algorithm. 

For stating the main result, we introduce a few more notations. We define the \textit{confusion set} $B(\theta^{o})$ and  $C(\theta^{o})$ as,
\begin{align*}
B(\theta^{o}) &:= \{\theta \in \Theta : a^*(\theta) \neq a^*(\theta^{o})~ \text{and} ~ \mu_{a^*(\theta^{o})}(\theta^{o}) = \mu_{a^*(\theta^{o})}(\theta) \},  \\
C(\theta^{o}) &:= \{a^{*}(\theta) : \theta \in B(\theta^{o})\}. 
\end{align*}

Intuitively, $B(\theta^{o})$ is the set of parameters that can be confused with the true parameter $\theta^{o}$. If $B(\theta^{o})$  is non-empty,  selecting $a^*(\theta^{o})$   and estimating the empirical mean is not sufficient to identify the true parameter because the same mean reward can result from other parameters in $B(\theta^{o})$. So, if $B(\theta^{o})$  is non-empty, more exploration (i.e., selecting sub-optimal arms other than $a^*(\theta^{o})$) is necessary to identify the true parameter. This exploration will contribute to the regret. On the other hand, if $B(\theta^{o})$  is empty, optimal parameter can be identified with much less exploration, which results in a bounded regret. $C(\theta^{o})$ is the corresponding  set of  arms that needs to be explored sufficiently for identifying the optimal parameter. So, whether $B(\theta^{o})$ is empty or non-empty is the structural condition that decides the performance of the algorithm.

We make the following assumption.

\begin{assumption}[Unique best action] \label{assumption:unique}
	For all $\theta\in\Theta$, the optimal action, $a^*(\theta)$, is unique.
\end{assumption}
We note that this is a standard assumption in the literature. This assumption can be removed at the expense of more notations. We define  $\Delta_{i}$ as,
\begin{align}
\Delta_{i} &:= \mu_{a^{*}(\theta^{o})}(\theta^{o}) - \mu_{i}(\theta^{o}), \label{eq:delta_i}
\end{align}
which is the difference between the mean value of the optimal arm and the mean value of arm $i$ for the true parameter $\theta^{o}$. This is the standard optimality gap notion used  in the  MAB literature \cite{auer2002finite}. Without loss of generality assume natural logarithms.

For each arm in $i \in C(\theta^{o})$, we define, 
\begin{align}
\beta_{i} &:= \min_{\theta:\theta\in B(\theta^o),a^*(\theta)=i} |\mu_{i}(\theta^{o}) - \mu_{i}(\theta) |. \label{eq:beta_i}
\end{align}

We use the following Lemma to compare our result with classical MAB result. The proof for this lemma is given in the appendix.

\begin{lemma}\label{lemma_beta_delta}
Let $\Delta_{i}$ and $\beta_{i}$ be as defined in \eqref{eq:delta_i} and \eqref{eq:beta_i} respectively. Then, for each $i \in C(\theta^{o})$, $\beta_{i} > 0$. Moreover,  $\beta_{i} > \Delta_{i}$.
\end{lemma}

\begin{algorithm}
	\caption{FP-UCB }	
	\label{alg:fp-ucb}
	\begin{algorithmic}[1]
	   \STATE Initialization: Select each arm in the set $A$ once
		\STATE Initialize episode number $k =1$, time step $t=|A|+1$
		\WHILE {$t \leq T$}
		\STATE $t_k=t-1$
		\STATE Compute the set 
		\[A_k = \left\{\begin{array}{ll} a^{*}(\theta), \theta \in \Theta : \forall i\in A,  ~ |\hat{\mu}_{i}(t_k) - \mu_{i}(\theta) | \leq \sqrt{\frac{3\log (k)}{n_{i}(t_k)}} \end{array} \right\}\]
		\IF {$|A_k| \neq 0$}
		\STATE Select each arm in the set $A_k$ once %, in a round robin manner
		\STATE $t \leftarrow t + |A_k|$
		\ELSE
		\STATE Select each arm in the set $A$ once %, in a round robin manner
		\STATE $t \leftarrow t + |A|$
		\ENDIF
		\STATE $k\leftarrow k+1$
		\ENDWHILE
	\end{algorithmic}
\end{algorithm}

 \begin{figure}[t]
	\centering
	\includegraphics[width=\linewidth]{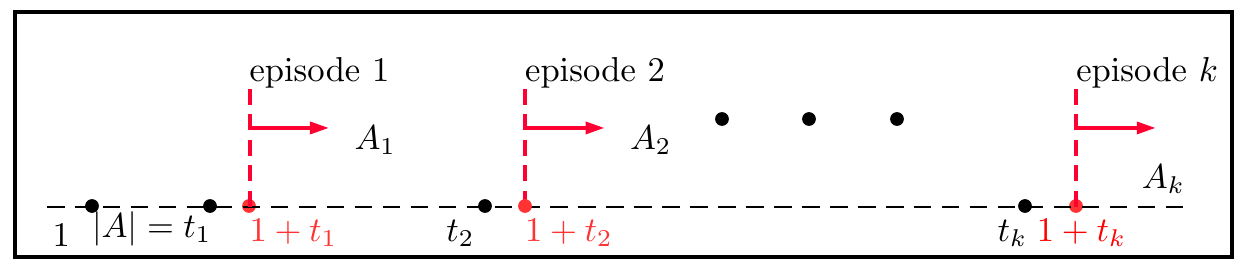}
	\captionof{figure}{An illustration of the episodes and time slots of the FP-UCB algorithm.}
	\label{fig:episodes}
\end{figure} We now present the finite time performance guarantee for our FP-UCB  algorithm.

\begin{theorem}
	\label{thm:fp-ucb}
	Under the FP-UCB algorithm,
	\begin{align}
	\label{eq:thm1_1}
	\mathbb{E}[R(T)] \leq D_1,~~\text{if}~ B(\theta^{o})~\text{empty, and}~ 
%	\label{eq:thm1_2}
	\mathbb{E}[R(T)] \leq D_2 + 12 \log(T) \sum_{i\in C(\theta^{o})} \frac{\Delta_i}{\beta^2_{i}},~~ \text{if}~ B(\theta^{o})~\text{non-empty},
	\end{align}
	where $D_1$ and $D_{2}$ are problem dependent constants that depend only on the problem parameters $|A|$  and $(\mu_{i}(\theta), \theta \in \Theta)$, but do not depend  on  $T$.
\end{theorem}

\begin{remark}[Comparison with the classical MAB results]
	Both UCB type algorithms and Thompson Sampling type algorithms give a problem dependent regret bound $O(\log T)$. More precisely, assuming that the optimal arm is arm 1, the regret of the UCB algorithm, $\mathbb{E}[R_{\textnormal{UCB}}(T)]$, is given by \cite{auer2002finite}
	\begin{align*}
	&\mathbb{E}[R_{\textnormal{UCB}}(T)] = O\left(\sum^{L}_{i=2} \frac{1}{\Delta_{i}} \log T\right).
	\end{align*}
	
	On the other hand, FP-UCB algorithm achieves the regret 
	\begin{align*}
	\mathbb{E}[R_{\textnormal{FP-UCB}}(T)] = O(1) ,~~\text{if}~ B(\theta^{o})~\text{empty, and} ~
%	\hspace{2cm} 
	O\left(   \sum_{i\in C(\theta^{o})} \frac{\Delta_i}{\beta^2_{i}} \log T\right), \text{if}~ B(\theta^{o})~\text{non-empty.} 
	\end{align*}
	
	Clearly, for some MAB problems, FP-UCB algorithm achieves a bounded regret ($O(1)$) as opposed to the increasing regret  ($O(\log T)$) of the standard UCB algorithm. Even in the cases where FP-UCB algorithm incurs an increasing regret ($O(\log T)$), the preceding constant ($\Delta_{i}/\beta^{2}_{i}$) is smaller than  the preceding constant  ($1/\Delta_{i}$) of the standard UCB algorithm because $\beta_{i} > \Delta_{i}$.
\end{remark}

We now give the asymptotic  lower bound for the finitely parameterized multi-armed bandits problem from \cite{rajeev1989asymptotically}, for comparing the  performance of our FP-UCB algorithm.
%Let us call a scheme uniformly good if for parameter $\theta^o \in \Theta$ \[ \E[R(T)] = o(T^a) \hspace*{1cm} \text{for every }a>0. \]
\begin{theorem}[Lower bound \cite{rajeev1989asymptotically}]
\label{thm:LB}
For any uniformly good control scheme under the parameter $\theta^o$,
\begin{align*}
\liminf_{T\to\infty} \frac{\mathbb{E}[R(T)]}{\log(T)} \geq \min_{h \in H}  \max_{\theta \in B(\theta^{o})} \frac{ \sum_{u\in A\setminus \{a^{*}(\theta^{o})\}} h_u (\mu_{a^{*}(\theta^{o})}(\theta^{o}) - \mu_{u}(\theta^{o}))}{   \sum_{u\in A\setminus \{a^{*}(\theta^{o})\}} h_u     D_{u}( \theta^o \| \theta)}. \nonumber
\end{align*}
where $H$ is a probability simplex with $|A|-1$ vertices and, for any $u\in A\setminus \{a^{*}(\theta^{o})\}$, $D_{u}( \theta^o \| \theta) = \int P_{u}(x;\theta^o) \log ( P_{u}(x;\theta^o) / P_{u}(x;\theta)) dx$ is the  KL-divergence between the probability distributions $P_{u}(\cdot;\theta^o)$ and $P_{u}(\cdot;\theta)$.

%\begin{align*}
%\liminf_{T\to\infty} \frac{\mathbb{E}[R(T)]}{\log(T)} \geq  \max_{\theta \in B(\theta^{o})} \frac{ \mu_{a^{*}(\theta^{o})}(\theta^{o}) - \mu_{a^*(\theta)}(\theta^{o})}{D_{a^*(\theta)}( \theta^o \| \theta)}. \nonumber
%\end{align*}
%where $D_{a^*(\theta)}( \theta^o \| \theta) = \int P_{a^*(\theta)}(x;\theta^o) \log ( P_{a^*(\theta)}(x;\theta^o) / P_{a^*(\theta)}(x;\theta)) dx$ is the  KL-divergence between the probability distributions $P_{a^*(\theta)}(\cdot;\theta^o)$ and $P_{a^*(\theta)}(\cdot;\theta)$.  
	
\begin{remark}[Optimality of the FP-UCB algorithm]
From Theorem  \ref{thm:LB}, the achievable  regret of any multi-armed bandits learning algorithm is lower bounded by $\Omega(1)$  when $B(\theta^{o})$ is empty, and $\Omega(\log T)$ when $B(\theta^{o})$ is non-empty. Our FP-UCB algorithm achieves these bounds and hence achieves the order optimal performance.  
\end{remark}
	
\end{theorem}

\section{Analysis of the FP-UCB Algorithm}\label{section:ucb-analysis}

In this section, we give the proof of Theorem \ref{thm:fp-ucb}. For reducing the notation, without loss of generality we assume that the true optimal arm is arm $1$, i.e., $a^*=a^*(\theta^{o})=1.$ We will also denote  $\mu_{j}(\theta^{o})$ as $\mu^{o}_{j}$, for any $j \in A$.

Now, we can rewrite the expected regret from \eqref{eq:regret-defn1} as 
\begin{align*}
&\E[R(T)] = \E\left[\sum^{T}_{t=1} (\mu^{o}_{1} -\mu^{o}_{a(t)} )\right] \\
&= \sum_{i=2}^L \Delta_i ~ \E\left[\sum_{t=1}^{T} \mathbbm{1}\{a(t) = i\} \right] = \sum_{i=2}^L \Delta_i ~ \E\left[n_i(T)\right].
\end{align*}Since the algorithm selects arms only from the set $A$, this can be written as \begin{align}
\label{eq:regret-defn2}
\E[R(T)] =  \sum_{i\in A} \Delta_i ~ \E\left[n_i(T)\right].
\end{align}

We first prove the following important propositions. %, which gives the regret for $i\in A\setminus C(\theta^{o})$.
\begin{proposition}
	\label{prop:1}
	For all $i \in A\setminus C(\theta^{o}), i \neq 1$, under FP-UCB algorithm,
	\begin{align}
	\E\left[ n_i(T)\right] \leq C_i,
	\end{align}
	where $C_i$ is a problem dependent constant that does not depend on $T$. 
\end{proposition}
\begin{proof} 
	Consider an arm $i \in A\setminus C(\theta^{o}), i \neq 1$. Then, by definition, there exists a $\theta \in \Theta$ such that $a^{*}(\theta) = i$. Fix a $\theta$ which satisfies this condition. Define
	\begin{align*}
	\alpha_1(\theta) :=  |\mu_{1}(\theta^{o}) - \mu_{1}(\theta)|.
	\end{align*}
	It is straightforward to note that when $i \in A\setminus C(\theta^{o})$, then the $\theta$ which we considered above is not in $B(\theta^o)$. Hence, by definition, $\alpha_{1}(\theta) > 0$. 
	
	%
	%$ i \neq 1$, $\theta\notin B(\theta^o)$ from definition, and hence $\alpha_1(\theta) >  0, \forall \theta \in \Theta \setminus \{\theta^{o}\}$. 
	%	It is straightforward to note that when $B(\theta^{o})$ is empty, $\alpha_1(\theta) >  0, \forall \theta \in \Theta \setminus \{\theta^{o}\}$.  
	
	For notational convenience, we will denote $\mu_{j}(\theta)$ simply as $\mu_{j}$, for any $j \in A$. Notice that the algorithm picks $i^{\text{th}}$ arm once in $t\in\{1,\ldots,|A|\}$. Define $K_T$ (note that this is a random variable) to be the total number of episodes in time horizon $T$ for the FP-UCB algorithm. It is straightforward that $K_T \leq T$. Now, 
	\begin{align}
	\E[ n_i(T)]  &=  1 + \E\left[\sum_{t=|A|+1}^{T} \mathbbm{1}\{a(t) = i\} \right] \nonumber\\
	& \stackrel{(a)}{=}  1 + \E\left[\sum_{k=1}^{K_T} (\mathbbm{1}\{i\in A_k\} + \mathbbm{1}\{A_k=\varnothing\}) \right] \nonumber\\
	& \leq  1+  \sum_{k=1}^T \left[\Pp\left(\left\{i\in A_k  \right\}\right) + \Pp\left(\left\{A_k=\varnothing \right\} \right)\right] 	\label{eq:first_partition} \\
	& =  1+  \sum_{k=1}^T \left[ \Pp\left(\left\{i\in A_k, 1 \in A_k  \right\}\right)+ \Pp\left(\left\{i\in A_k, 1 \notin A_k  \right\}\right) + \Pp\left(\left\{A_k=\varnothing \right\}\right)\right] \nonumber \\
	& \leq 1+  \sum_{k=1}^T \left[ \Pp\left(\left\{i\in A_k, 1 \in A_k  \right\}\right)+ \Pp\left(\left\{i\in A_k, 1 \notin A_k  \right\}\right)  + \Pp\left(\left\{i\notin A_k, 1 \notin A_k \right\}\right)\right] \nonumber \\
	\label{eq:first_split}
	& \leq  1+  \sum_{k=1}^T  [\Pp (\left\{i\in A_k, 1 \in A_k  \right\}) + \Pp (\left\{1 \notin A_k  \right\}) ].
	\end{align} 
	Here (a) follows from the algorithm definition.
	% and (b) follows since $a^*(\theta')\notin A_k$ for all $\theta'\in\Theta.$

	We will first analyze the second summation term in \eqref{eq:first_split}. First observe that, we can write $n_{j}(t_k) = 1+ \sum^{k-1}_{\tau=1} (\mathbbm{1}\{j \in A_\tau \} + \mathbbm{1}\{A_\tau = \varnothing \}) $ for any $j\in A$ and episode $k$. Thus, $n_{j}(t_k)$ lies between 1 and $k$. Now,
	\begin{align}
	\sum_{k=1}^T  \Pp (\left\{1 \notin A_k  \right\}) 
	&\stackrel{(a)}{=}  \sum_{k=1}^T  \Pp \left( \bigcup_{j\in A} \left\{|\hat{\mu}_j(t_k)-\mu^o_j| > \sqrt{\frac{3\log k}{n_j(t_k)}}\right\}  \right) \nonumber \\
	& \stackrel{(b)}{\leq}  \sum_{k=1}^T  \sum_{j\in A} \Pp \left( |\hat{\mu}_j(t_k)-\mu^o_j| >\sqrt{\frac{3\log k}{n_j(t_k)}} \right) \nonumber \\
	& \stackrel{(c)}{=}  \sum_{k=1}^T  \sum_{j\in A} \Pp \left( \left|\frac{1}{n_j(t_k)} \sum_{\tau=1}^{n_j(t_k)} X_{j}(\tau) -\mu^o_j \right| >\sqrt{\frac{3\log k}{n_j(t_k)}} \right) \nonumber \\
	& \stackrel{(d)}{\leq}  \sum_{k=1}^T  \sum_{j\in A} \sum_{m=1}^{k} \Pp \left( \left|\frac{1}{m} \sum_{\tau=1}^{m} X_{j}(\tau) -\mu^o_j \right| >\sqrt{\frac{3\log k}{m}} \right) \nonumber \\
	& \stackrel{(e)}{\leq}  \sum_{k=1}^T  \sum_{j\in A}  \sum_{m=1}^{k} 2\exp\left( -2 m {\frac{3\log k}{m}}  \right) =  \sum_{k=1}^T  \sum_{j\in A} 2k^{-5} \leq 4 |A| . \label{eq:term21}
	\end{align}
	Here (a) follows from algorithm definition, (b) from the union bound, and (c) from the definition in \eqref{eq:muhat}. Inequality (d) follows by conditioning the random variable $n_{j}(t_k)$ that lies between 1 and $k$ for any $j\in A$ and episode $k$. Inequality (e) follows from Hoeffding's inequality \cite[Theorem 2.2.6]{Vershynin}.
	
	For analyzing the first summation term in \eqref{eq:first_split}, define the event $E_k :=  \left\{ n_1(t_k) < {12\log k}/{\alpha^{2}_1(\theta)} \right\}.$ Denote the complement of this event as $E^{c}_k$. Now the first summation term in \eqref{eq:first_split} can be written as  \begin{align}
	&\hspace{-1cm}\sum_{k=1}^{T}  \Pp (\left\{i\in A_k, 1 \in A_k  \right\})  \nonumber \\
	\label{eq:AE-t1}
	&\hspace{-0.5cm} = \sum_{k=1}^{T}  \Pp (\left\{i\in A_k, 1 \in A_k, E^c_k  \right\})  \\
	\label{eq:AE-t2}
	&\hspace{0.5cm} + \sum_{k=1}^{T}  \Pp (\left\{i\in A_k, 1 \in A_k, E_k \right\}).
	\end{align}    
	
	Analyzing \eqref{eq:AE-t1}, we get,
	\begin{align}
	\Pp (\left\{i\in A_k, 1 \in A_k, E^{c}_k  \right\}) 
	&=  \Pp \left(  \bigcap_{j\in A} \{|\hat{\mu}_j(t_k)-\mu^o_j| < \sqrt{\frac{3\log k}{n_j(t_k)}} \} \bigcap_{j\in A} \{ |\hat{\mu}_j(t_k)-\mu_j| < \sqrt{\frac{3\log k}{n_j(t_k)}} \} \bigcap  E^{c}_k \right) \nonumber  \\
	\label{eq:zeroprob1}
	&\leq \Pp \left( \{|\hat{\mu}_1(t_k)-\mu^o_1| < \sqrt{\frac{3\log k}{n_1(t_k)}} \},  |\{ \hat{\mu}_1(t_k)-\mu_1| < \sqrt{\frac{3\log k}{n_1(t_k)}} \}, E^{c}_k \right) = 0.
	\end{align}
	This is because the events $\{|\hat{\mu}_1(t_k)-\mu^o_1| < \sqrt{\frac{3\log k}{n_1(t_k)}}\}$ and $\{|\hat{\mu}_1(t_k)-\mu_1| < \sqrt{\frac{3\log k}{n_1(t_k)}}\}$ are disjoint under $E_k^c$, that is, when $n_1(t_k) \geq {12\log (k)}/{\alpha^2_1(\theta)}$. To see this, notice that
	\begin{align*}
	\left\{|\hat{\mu}_1(t_k)-\mu^o_1| < \sqrt{\frac{3\log k}{n_1(t_k)}}\right\} &\subseteq \left\{|\hat{\mu}_1(t_k)-\mu^o_1| < \frac{\alpha_{1}(\theta)}{2}\right\}, \\
	\left\{|\hat{\mu}_1(t_k)-\mu_1| < \sqrt{\frac{3\log k}{n_1(t_k)}}\right\} &\subseteq
	\left\{|\hat{\mu}_1(t_k)-\mu_1| < \frac{\alpha_{1}(\theta)}{2}\right\},
	\end{align*} 
	for $n_1(t_k) \geq {12\log k}/{\alpha^2_1(\theta)}$. Moreover, since $|\mu^{o}_{1} - \mu_{1}| = \alpha_{1}(\theta)$, $\{|\hat{\mu}_1(t_k)-\mu^o_1| < {\alpha_{1}(\theta)}/{2}\}$ and $\{|\hat{\mu}_1(t_k)-\mu_1| < {\alpha_{1}(\theta)}/{2}\}$ are disjoint sets. Hence, their subsets are also disjoint.

	For analyzing \eqref{eq:AE-t2}, define $n'_{1}(t_k) := 1+ \sum^{k-1}_{\tau=1} \mathbbm{1}\{1 \in A_\tau \} $. Note that, according to the FP-UCB algorithm, arm $1$ can be selected if $A_\tau$ is empty as well, so $n'_{1}(t_k) \leq n_{1}(t_k)$.  Define $k_{i}(\theta)$ and $m(k)$ as,
	\begin{align} \label{eq:ktheta}
	&\hspace*{-0.3cm}k_{i}(\theta) := \min \left\{k:  k \geq 3,  k>\lceil 12\log (k)/\alpha^{2}_1(\theta) \rceil  \right\}, \\
	&\hspace*{-0.3cm}m(k) := \max\{1,  k-\lceil 12\log (k)/\alpha^{2}_1(\theta) \rceil \}.
	\end{align}
	Note that $k_{i}(\theta)$ is a problem dependent constant and does not depend on $T$. Also, $m(k) = k-\lceil 12\log (k)/\alpha^{2}_1(\theta) \rceil$ for all $k \geq k_{i}(\theta)$. We claim that for all $k \geq k_{i}(\theta)$,
	\begin{align}
	\label{eq:claim}
	&\left\{n'_1(t_k) < {12\log (k)}/{\alpha^{2}_{1}(\theta)} \right\} \subseteq \left\{1 \notin A_\tau, \text{for some } \tau,  m(k) \leq \tau \leq k-1 \right\}. 
	\end{align} 
	To see this, suppose there exists no $ \tau,$  $m(k) \leq \tau \leq k-1$, such that $1 \notin A_\tau$. Then, $1 \in  A_\tau$ for all $ \tau,$ where $m(k) \leq \tau \leq k-1$. So, by definition $n'_{1}(t_k) \geq (k-m(k))  = \lceil 12\log (k)/\alpha^{2}_1(\theta) \rceil$ for $k \geq k_{i}(\theta)$. So, the complement of the RHS  of \eqref{eq:claim} is a subset of the complement of the LHS  of \eqref{eq:claim}. Hence the claim follows.

	Now, 
	\begin{align}
	\sum_{k=1}^{T}  \Pp (\left\{i\in A_k, 1 \in A_k, E_k  \right\})  &\leq \nonumber \sum_{k=1}^{T}  \Pp (E_k  ) \\ &
	\stackrel{(a)}{\leq}  \sum_{k=1}^T \Pp \left( n'_1(t_k) < {12\log (k)}/{\alpha^2_{1}(\theta)} \right) \nonumber \\
	&\stackrel{(b)}{\leq}  k_{i}(\theta) +   \sum_{k=k_{i}(\theta)}^T \Pp ( n'_1(t_k) < {12\log (k)}/{\alpha^2_{1}(\theta)} ) \nonumber \\
	&\stackrel{(c)}{\leq}   k_{i}(\theta) + \sum_{k=k_{i}(\theta)}^T \Pp \left(  \left\{1 \notin A_\tau, \text{for some } \tau,  m(k) \leq \tau \leq k-1 \right\} \right) \nonumber \\
	&\stackrel{(d)}{=}   k_{i}(\theta) + \nonumber \sum_{k=k_{i}(\theta)}^T \Pp \left( \bigcup_{\tau=m(k)}^{k-1} \bigcup_{j\in A} |\hat{\mu}_j(\tau)-\mu^o_j| > \sqrt{\frac{3\log \tau}{n_{j}(t_\tau)}} \right)  \nonumber \\
	&\leq   k_{i}(\theta)+  \sum_{k=k_{i}(\theta)}^T  \sum^{k-1}_{\tau = m(k)} \sum_{j\in A} \Pp \left( |\hat{\mu}_j(\tau)-\mu^o_j| > \sqrt{\frac{3\log \tau}{n_{j}(t_\tau)}} \right) \nonumber \\
	&\stackrel{(e)}{\leq}   k_{i}(\theta) +  \sum_{k=k_{i}(\theta)}^T  \sum^{k-1}_{\tau = m(k)}  \frac{2|A|}{\tau^{5}} \\
	&\leq   k_{i}(\theta) +  \sum_{k=k_{i}(\theta)}^T  \frac{2|A| k}{(m(k))^{5}} \nonumber \\
	&=  k_{i}(\theta) +  \sum_{k=k_{i}(\theta)}^T  \frac{2|A| k}{(   k-\left\lceil \frac{12\log (k)}{\alpha^{2}_{1}(\theta)} \right\rceil  )^{5}} 
	\stackrel{(f)}{=}   k_{i}(\theta) + K_i(\theta), \label{eq:Ktheta}
	\end{align}
	where $K_i(\theta)$ is a problem dependent constant that does not depend on $T$. 
	
	In the above analysis, (a) follows from the definition of $E_k$ and the observation that $n'_{1}(t_k) \leq n_{1}(t_k)$. Considering $T$ to be greater than or equal to $k_i(\theta)|A|$, equality (b) follows; note that this is an artifact of the proof technique and does not affect the theorem statement since $\E[n_i(T')]$, for any $T'$ less than $k_i(\theta)|A|$, can be trivially upper bounded by $\E[n_i(T)]$. Inequality (c) follows from \eqref{eq:claim}, (d) by the FP-UCB algorithm, (e) is similar to the analysis in \eqref{eq:term21}, and (f) follows from the fact that  $k>\lceil 12\log (k)/\alpha^{2}_{1}(\theta) \rceil $ for all $k\geq k_{i}(\theta)$.
	
	Now, using \eqref{eq:Ktheta} and \eqref{eq:zeroprob1} in \eqref{eq:AE-t1} and \eqref{eq:AE-t2}, we get,
	\begin{align}
	\label{eq:putback1}
	\sum_{k=1}^{T}  \Pp (\left\{i\in A_k, 1 \in A_k  \right\})  \leq  k_{i}(\theta) + K_i(\theta).
	\end{align} 
	Using \eqref{eq:putback1} and \eqref{eq:term21} in \eqref{eq:first_split}, we get,
	\begin{align*}
	\mathbb{E}[n_{i}(T)] \leq C_i,
	\end{align*}
	where $C_i = 1+4|A|+\min_{\theta:a^*(\theta)=i} ( k_{i}(\theta) + K_i(\theta))$, which is a problem dependent constant that does not depend on $T$. This concludes the proof.
\end{proof}

 \begin{proposition} 
\label{prop:2}
For any $i\in C(\theta^o)$, under FP-UCB algorithm,	
 	\begin{align}
 	\E\left[ n_i(T)\right] \leq 2+4|A|+ \frac{12 \log(T)}{\beta^{2}_{i}}.
 	\end{align}
 \end{proposition}
 \begin{proof}
 	Fix an $i \in C(\theta^{o})$. Then there exists a $\theta \in B(\theta^{o})$ such that $a^*(\theta)=i$. Fix a $\theta$ which satisfies this condition. Define the event $F(t) :=  \left\{ n_i(t-1) < {12\log T}/{\beta_{i}^{2}} \right\}.$
 	Now, \begin{align}
 	\E[ n_i(T)]  &= 1+ \E\left[\sum_{t=|A|+1}^{T} \mathbbm{1}\{a(t) = i\} \right] \nonumber \\
 	& = 1+ \E\left[\sum_{t=|A|+1}^{T} \mathbbm{1}\{a(t) = i,F(t)\} \right]  +  \E\left[\sum_{t=|A|+1}^{T} \mathbbm{1}\{a(t) = i,F^c(t)\} \right]. \label{eq:second_part}
 	\end{align} 
 	Analyzing the first summation term in \eqref{eq:second_part} we get,  \begin{align}
 	\E\left[\sum_{t=|A|+1}^{T} \mathbbm{1}\{a(t) = i,F(t)\} \right] &= \E\left[\sum_{t=|A|+1}^{T} \mathbbm{1}\{a(t) = i\} \mathbbm{1}\left\{ n_i(t-1) < {12\log T}/{\beta_{i}^{2}} \right\} \right] \nonumber \\
 	&\leq 1+{12\log T}/{\beta_{i}^{2}}. \label{eq:final_1}
 	\end{align} 
% 	This is true because $n_i(t) < {12 \log T}/{\beta_{i}^{2}} $ for all $t$ from $|A|$ to $T-1$, in particular for $t=T-1$ we have $n_i(T-1) < {12 \log T}/{\beta_{i}^{2}} $. Since $n_i(T)\leq 1+ n_i(T-1)$, \eqref{eq:final_1} follows.
 	
 	We use the same decomposition as in the proof of Proposition \ref{prop:1} for the second summation term in \eqref{eq:second_part}. Thus we get, \begin{align}
 	\E\left[\sum_{t=|A|+1}^{T} \mathbbm{1}\{a(t) = i,F^c(t)\} \right] &=
 	 \E\left[\sum_{k=1}^{K_T} \mathbbm{1}\{i\in A_k, F^c(t_k + 1)\} + \mathbbm{1}\{A_k=\varnothing, F^c(t_k + 1)\} \right] \nonumber\\
 	& \label{eq:second_split_1}  \leq \sum_{k=1}^{T}  \Pp (\left\{i\in A_k, 1 \in A_k, F^c(t_k + 1)  \right\}) + \\& \hspace{1.5cm} \sum_{k=1}^{T} \Pp (\left\{1 \notin A_k, F^c(t_k + 1)  \right\}), \label{eq:second_split_2}
 	\end{align} 
 	following the analysis in \eqref{eq:first_split}. First, consider \eqref{eq:second_split_2}. From the analysis in \eqref{eq:term21} we have \begin{equation}
 	\sum_{k=1}^{T}  \Pp (\left\{1 \notin A_k , F^c(t_k + 1)  \right\}) \leq \sum_{k=1}^{T}  \Pp (\left\{1 \notin A_k  \right\})  \leq 4|A|. \label{eq:final_2}
 	\end{equation}
 	For any $i\in A$ and episode $k$ under event $F^c(t_k+1)$, we have \[n_i(t_k) \geq \frac{12\log T}{\beta_{i}^{2}}\geq \frac{12\log t_k}{\beta_{i}^{2}} \geq \frac{12\log k}{\beta_{i}^{2}}\] since $t_k$ satisfies $k\leq t_k \leq T$. From \eqref{eq:beta_i}, it further follows that \[ \sqrt{\frac{3\log k}{n_j(t_k)}} \leq \frac{\beta_{i}}{2} \leq \frac{|\mu_i(\theta^o)-\mu_i(\theta)|}{2}. \] So, following the analysis in \eqref{eq:zeroprob1} for \eqref{eq:second_split_1}, we get \begin{align}
 	&\Pp (\left\{i\in A_k, 1 \in A_k, F^{c}(t_k+1)  \right\})  \nonumber  \\
 	&=  \Pp \left( \begin{array}{ll} \bigcap_{j\in A} \{|\hat{\mu}_j(t_k)-\mu_j(\theta^o)| < \sqrt{\frac{3\log k}{n_j(t_k)}} \}, \\ \bigcap_{j\in A} \{ |\hat{\mu}_j(t_k)-\mu_j(\theta)| < \sqrt{\frac{3\log k}{n_j(t_k)}} \},  F^{c}(t_k+1)\end{array} \right) \nonumber  \\
 	\label{eq:final_3}
 	&\leq \Pp \left(  \{ |\hat{\mu}_i(t_k)-\mu_i(\theta^o)| < \sqrt{\frac{3\log k}{n_i(t_k)}} \},   \{|\hat{\mu}_i(t_k)-\mu_i(\theta)| < \sqrt{\frac{3\log k}{n_i(t_k)}} \}, F^{c}(t_k+1) \right) = 0.
 	\end{align}
 	Using equations \eqref{eq:final_1}, \eqref{eq:final_2}, and \eqref{eq:final_3} in \eqref{eq:second_part}, we get
 	\begin{align*}
 	\mathbb{E}[n_{i}(T)] \leq 2+4|A|+ \frac{12 \log(T)}{\beta^{2}_{i}}.
 	\end{align*}
 	This completes the proof.
 \end{proof}
 
We now give the proof of our main theorem.
\begin{proof}({\bf of Theorem  \ref{thm:fp-ucb}}) \\

From \eqref{eq:regret-defn2},
\begin{align}
\mathbb{E}[R(T)] = \sum_{i \in A} \Delta_{i} \mathbb{E}[n_{i}(T)] = \sum_{i \in A \setminus C(\theta^{o})} \Delta_{i} \mathbb{E}[n_{i}(T)] + \sum_{i \in  C(\theta^{o})} \Delta_{i} \mathbb{E}[n_{i}(T)]. \label{eq:main_regret_split}
\end{align}

Whenever $B(\theta^o)$ is empty, notice that $C(\theta^o)$ is empty. So, using Proposition \ref{prop:1}, \eqref{eq:main_regret_split} becomes \begin{align*}
\mathbb{E}[R(T)] &= \sum_{i \in A} \Delta_{i} \mathbb{E}[n_{i}(T)] \leq \sum_{i \in A} \Delta_{i} C_i \leq |A| \max_{i \in A} \Delta_{i} C_i.
\end{align*}

Whenever $B(\theta^o)$ is non-empty, $C(\theta^o)$ is non-empty. Analyzing \eqref{eq:main_regret_split}, we get,
\begin{align*}
\mathbb{E}[R(T)] &= \sum_{i \in A \setminus C(\theta^{o})} \Delta_{i} \mathbb{E}[n_{i}(T)] + \sum_{i \in  C(\theta^{o})} \Delta_{i} \mathbb{E}[n_{i}(T)] \\
&\stackrel{(a)}{\leq} \sum_{i \in A \setminus C(\theta^{o})} \Delta_{i} C_i + \sum_{i \in  C(\theta^{o})} \Delta_{i} \mathbb{E}[n_{i}(T)] \\
&\stackrel{(b)}{\leq} \sum_{i \in A \setminus C(\theta^{o})} \Delta_{i} C_i + \sum_{i \in  C(\theta^{o})} \Delta_{i} \left(2+4|A|+ \frac{12 \log(T)}{\beta^{2}_{i}}\right)\\
&\leq |A| \max_{i \in A} \Delta_{i} (2+C_i + 4|A|)+ 12\log(T) \sum_{i \in C(\theta^{o})} \frac{\Delta_{i}}{\beta^{2}_{i}}. 
\end{align*} 
Here (a) follows from Proposition \ref{prop:1} and (b) from Proposition \ref{prop:2}. Setting \begin{align} \label{eq:d1d2}
D_1 := |A| \max_{i \in A} \Delta_{i} C_i \text{ and } 
D_2 := |A| \max_{i \in A} \Delta_{i} (2+C_i + 4|A|)
\end{align} proves the regret bounds in \eqref{eq:thm1_1} of the theorem.
\end{proof}

We now provide the following lemma to characterize the problem dependent constants $C_{i}$ given in Proposition \ref{prop:1}. The proof for this lemma is given in the appendix.
\begin{lemma}\label{lemma_c_i}
	Under the hypotheses in Proposition \ref{prop:1}, we have
	\begin{align*}
	C_i  &\leq 1+4|A|+\min_{\theta:a^*(\theta)=i} ( 2E_i(\theta) (E_i(\theta)+1)   |A| + 4|A| \alpha^{10}_{1}(\theta)), 
	\end{align*}
	\text{where}  $E_i(\theta) =  \max\{ 3, \lceil 144/\alpha^{4}_1(\theta) \rceil  \}$ and $ \alpha_{1}(\theta) = |\mu_{1}(\theta^{o}) - \mu_{1}(\theta)|$. 
\end{lemma}
Now, using the above lemma with \eqref{eq:d1d2}, we have a characterization of the problem dependent constants in Theorem \ref{thm:fp-ucb}.

\section{Simulations}
In this section, we present detailed numerical simulation to illustrate the performance of FP-UCB algorithm compared to the other standard multi-armed bandits algorithms.

We first consider a simple setting to illustrate intuition behind  FP-UCB algorithm.  Consider $\Theta = \{\theta^{1}, \theta^{2}\} $ with $[\mu_{1}(\theta^{1}) , \mu_{2}(\theta^{1})]= [0.9, 0.5]$ and  $[\mu_{1}(\theta^{2}) , \mu_{2}(\theta^{2})]= [0.2, 0.5]$. Consider the reward distributions $P_i,i=1,2$ to be Bernoulli. Clearly, $a^{*}(\theta^{1}) = 1$ and  $a^{*}(\theta^{2}) = 2$. 

Suppose the true parameter is $\theta^{1}$, i.e., $\theta^{o} = \theta^{1}$. Then, it is easy to note that, in this case $B(\theta^{o})$ is empty, and hence $C(\theta^{o})$ is empty. So, according to Theorem \ref{thm:fp-ucb}, FP-UCB will achieve an $O(1)$ regret. The performance of the algorithm for this setting is shown in Fig. \ref{fig:O1regret-1}. Indeed, the regret doesn't increase after some time steps, which shows the bounded regret property. 
We note that in all the figures, the regret is averaged over $10$ runs, with the thick line showing the average regret and the band around shows the $\pm 1$ standard deviation. 

Now, suppose the true parameter is $\theta^{2}$, i.e., $\theta^{o} = \theta^{2}$. In this case $B(\theta^{o})$ is non-empty. In fact, $B(\theta^{o}) = \theta^1$ and  $C(\theta^{o}) = 1$. So, according to Theorem \ref{thm:fp-ucb}, FP-UCB will achieve an $O(\log T)$ regret. The performance of the algorithm shown in Fig. \ref{fig:logregret-1} suggests the same. Fig. \ref{fig:logregret-2}  plots the regret scaled by $\log t$, and the curve converges to a constant value, confirming the $O(\log T)$ regret performance.

We consider a problem with 4 arms where the mean values for the arms (corresponding to the true parameter $\theta^{o}$) are  $\mu(\theta^{o}) = [0.6, 0.4, 0.3, 0.2 ]$.  Consider the parameter set  $\Theta$ such that $\mu(\theta)$ for any $\theta$ is a permutation of $\mu(\theta^{o})$. Note that the cardinality of the parameter set, $|\Theta| = 24$, in this case.
 It is straightforward to show that $B(\theta^{o})$ is empty for this case. We compare the performance of FP-UCB algorithm for this case with  two standard multi-armed bandits algorithms.  Fig.  \ref{fig:comparision-1} shows the performance of standard UCB algorithm and that of FP-UCB algorithm.    Fig.  \ref{fig:comparision-2} compares the performance of standard Thompson sampling algorithm with that of FP-UCB algorithm.  The standard bandits algorithm incurs an increasing regret, while FP-UCB achieves a bounded regret. For $\mu(\theta') = [0.4, 0.6, 0.3, 0.2 ]$, we have $a^*(\theta')=2.$ Now we give a typical value for the $k_2(\theta')$, defined in \eqref{eq:ktheta}, used in the proof. For this $\theta'$ we have $k_2(\theta')= \min \left\{k:  k \geq 3,  k>\lceil 12\log (k)/\alpha^{2}_1(\theta') \rceil  \right\}= \min \left\{k:  k \geq 3,  k>\lceil 12\log (k)/0.2^2 \rceil  \right\}=2326$ since $\alpha_1(\theta')=0.2.$ When the reward distributions are not necessarily Bernoulli, note that $k_i(\theta)$ is $3$ for any $\theta$ with $a^*(\theta)=i$ satisfying $\alpha_1(\theta)>2\sqrt{3/e}.$
 
As before assume that $\mu(\theta^{o}) = [0.6, 0.4, 0.3, 0.2 ]$. But consider a larger parameter set $\Theta$ such that for any $\theta \in \Theta$, $\mu(\theta) \in \{0.6, 0.4, 0.3, 0.2\}^{4}$. Note that, due to repetitions in the mean rewards for the arms, definition of $a^*(\theta)$ needs to be updated, and the algorithmic way is to pick the minimum arm index out of which are having the same mean rewards. For example, consider $\mu(\theta) = [0.5, 0.6, 0.6, 0.2 ]$, and so as per our new definition, $a^*(\theta)=2$. Even in this scenario, we have $B(\theta^o)$ to be empty. Thus, FP-UCB achieves an $O(1)$ regret rather than $O(\log(T))$ as opposed to standard UCB algorithm and Thompson sampling algorithm.

We now consider a case where FP-UCB incurs an increasing regret. We again consider a problem with 4 arms where the mean values for the arms are  $\mu(\theta^{o}) = [0.4, 0.3, 0.2, 0.2 ]$. But consider a larger parameter set $\Theta$ such that for any $\theta \in \Theta$, $\mu(\theta) \in \{0.6, 0.4, 0.3, 0.2\}^{4}$. Note that the cardinality of $\Theta$, $|\Theta| = 4^{4}$ in this case. It is easy to observe that $B(\theta^{o})$ is non-empty, for instance $\theta$ with mean arm values $[0.4,0.6,0.3,0.2]$ is in $B(\theta^o)$. %
Fig. \ref{fig:comparision-3} compares the performance of standard UCB and FP-UCB algorithms for this case. We see FP-UCB incurring $O(\log(T))$ regret here. Also note that the performance of the FP-UCB in this case also is superior to the standard UCB algorithm.
%Figures \ref{fig:comparision-3} and \ref{fig:comparision-4} shows the performance of standard UCB algorithm, Thompson sampling algorithm, and FP-UCB algorithm. Here the Thompson sampling algorithm seems to be performing better than FP-UCB, we suspect this is the cause of large parameter space. FP-UCB incurs a large initial regret since there are many of $4^4$ confidence balls capturing the average rewards, problem dependent constant $D_2$ of Theorem \ref{thm:fp-ucb} supports this suspicion.

\begin{figure*}[h!]
	\centering
	\begin{minipage}{.3\textwidth}
		\centering
		\includegraphics[width=\linewidth]{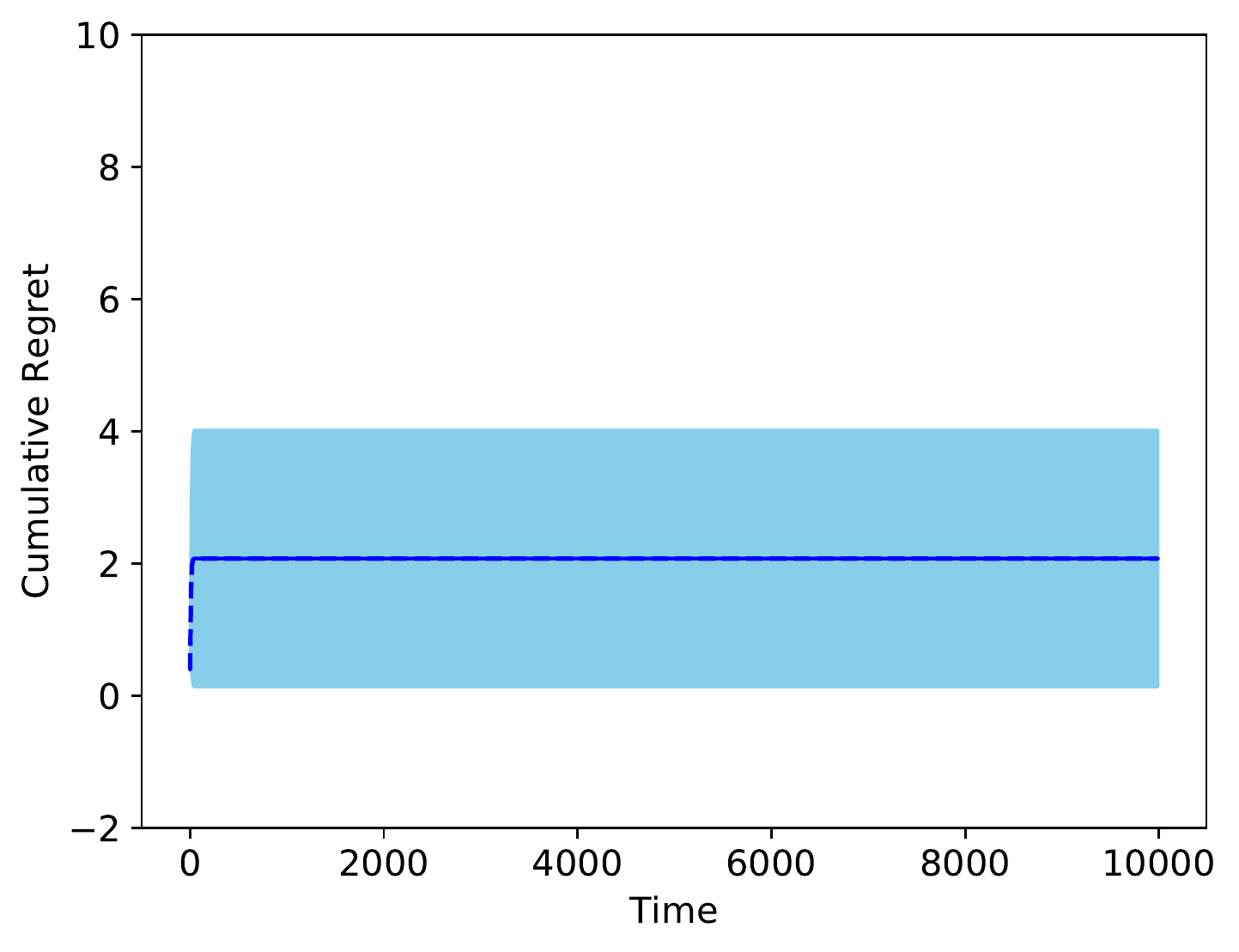}
		\captionof{figure}{}
		\label{fig:O1regret-1}
	\end{minipage}%
	\begin{minipage}{.3\textwidth}
		\centering
		\includegraphics[width=\linewidth]{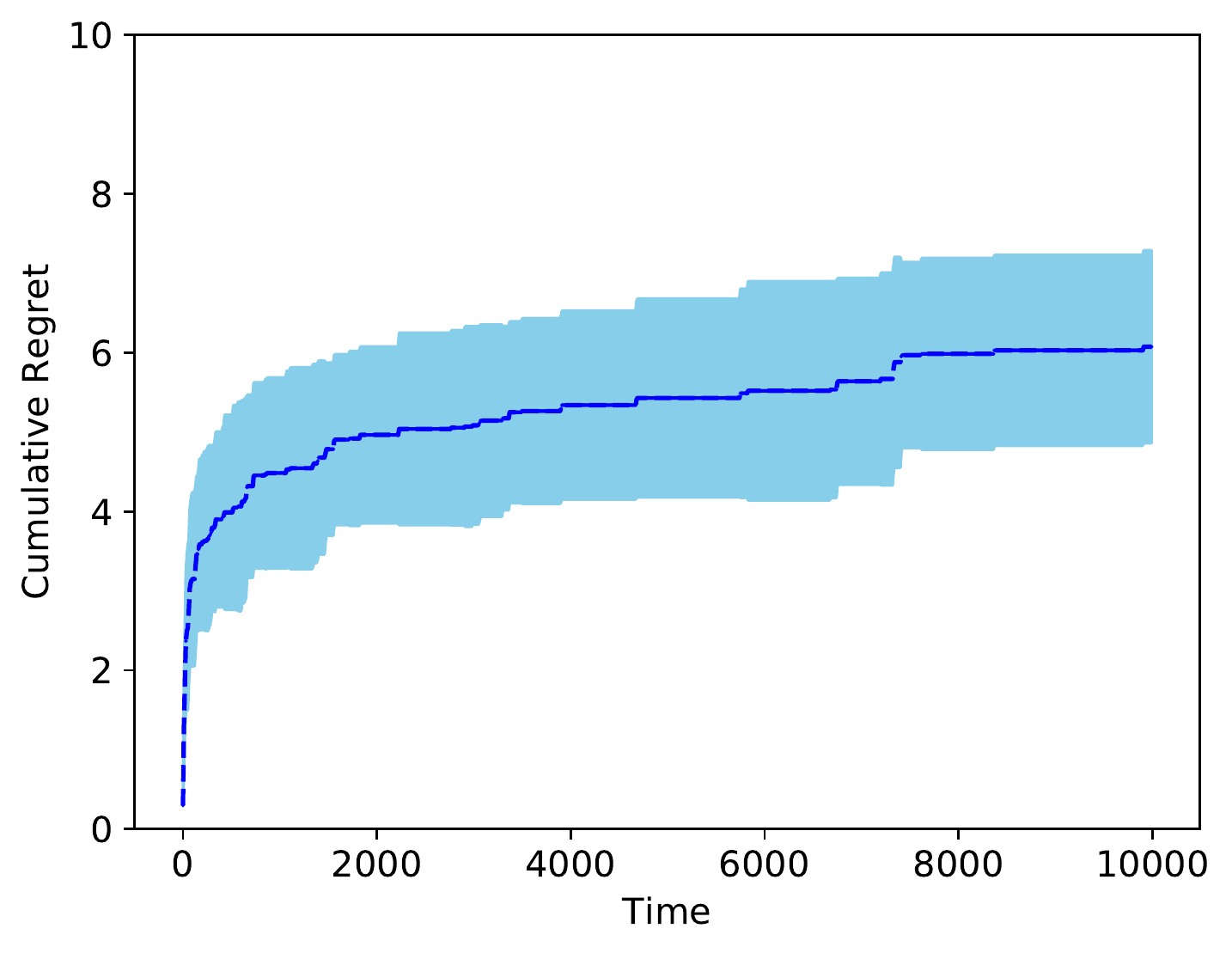}
		\captionof{figure}{}
		\label{fig:logregret-1}
	\end{minipage}
	\begin{minipage}{.3\textwidth}
		\centering
		\includegraphics[width=\linewidth]{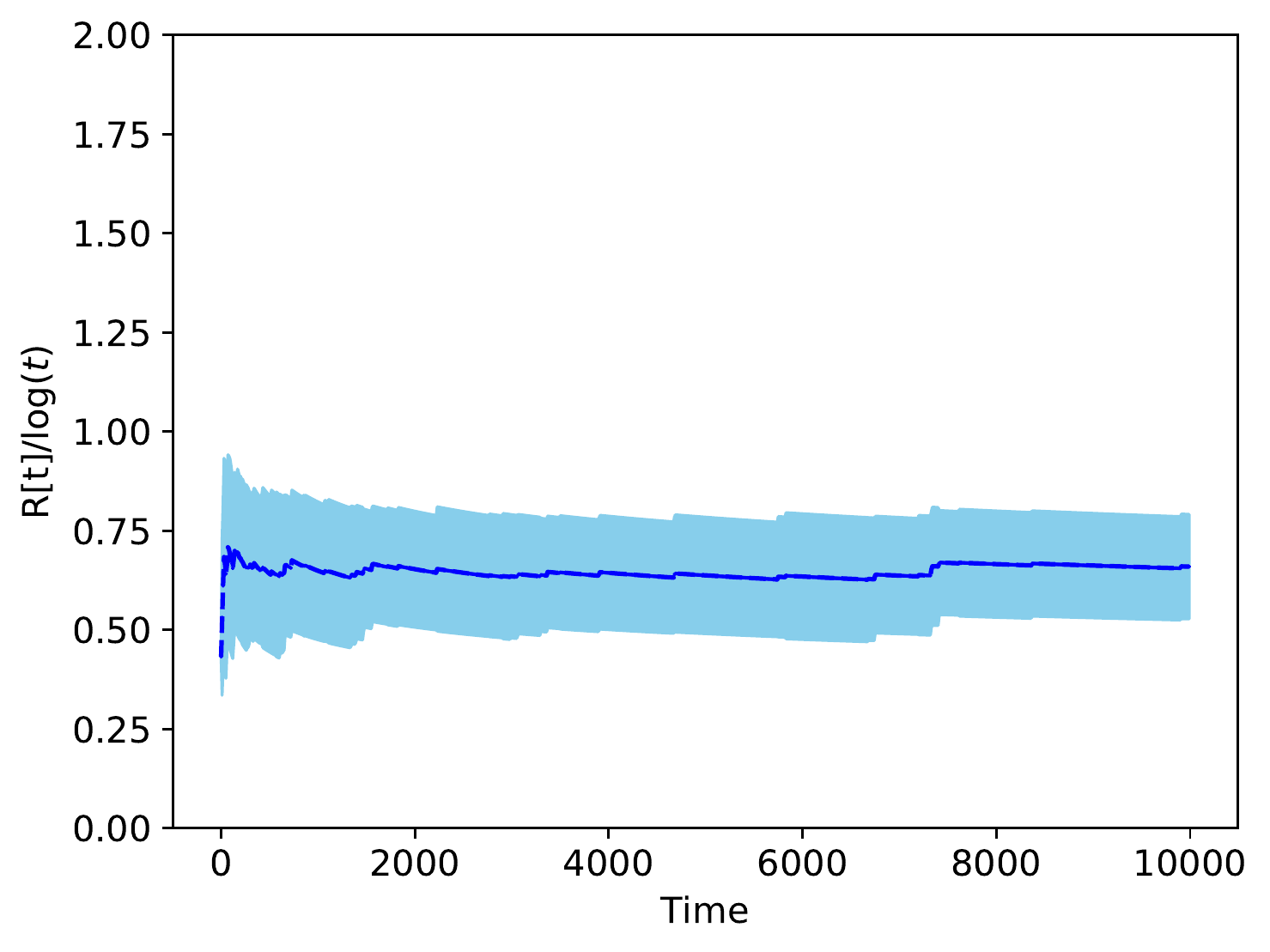}
		\captionof{figure}{}
		\label{fig:logregret-2}
	\end{minipage}
\end{figure*}

\begin{figure*}[h!]
	\centering
	\begin{minipage}{.3\textwidth}
		\centering
		\includegraphics[width=\linewidth]{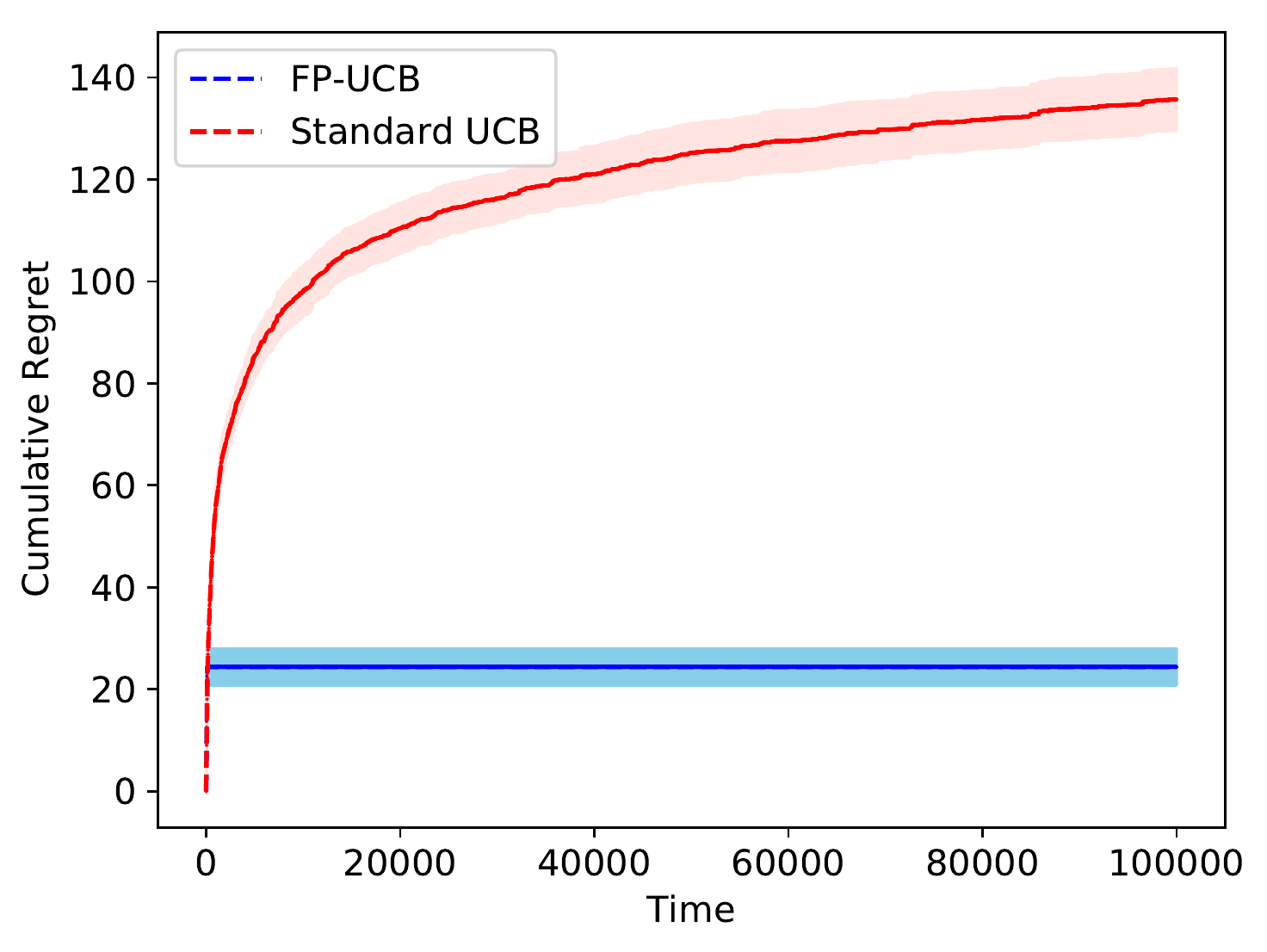}
		\captionof{figure}{}
		\label{fig:comparision-1}
	\end{minipage}
	\begin{minipage}{.3\textwidth}
		\centering
		\includegraphics[width=\linewidth]{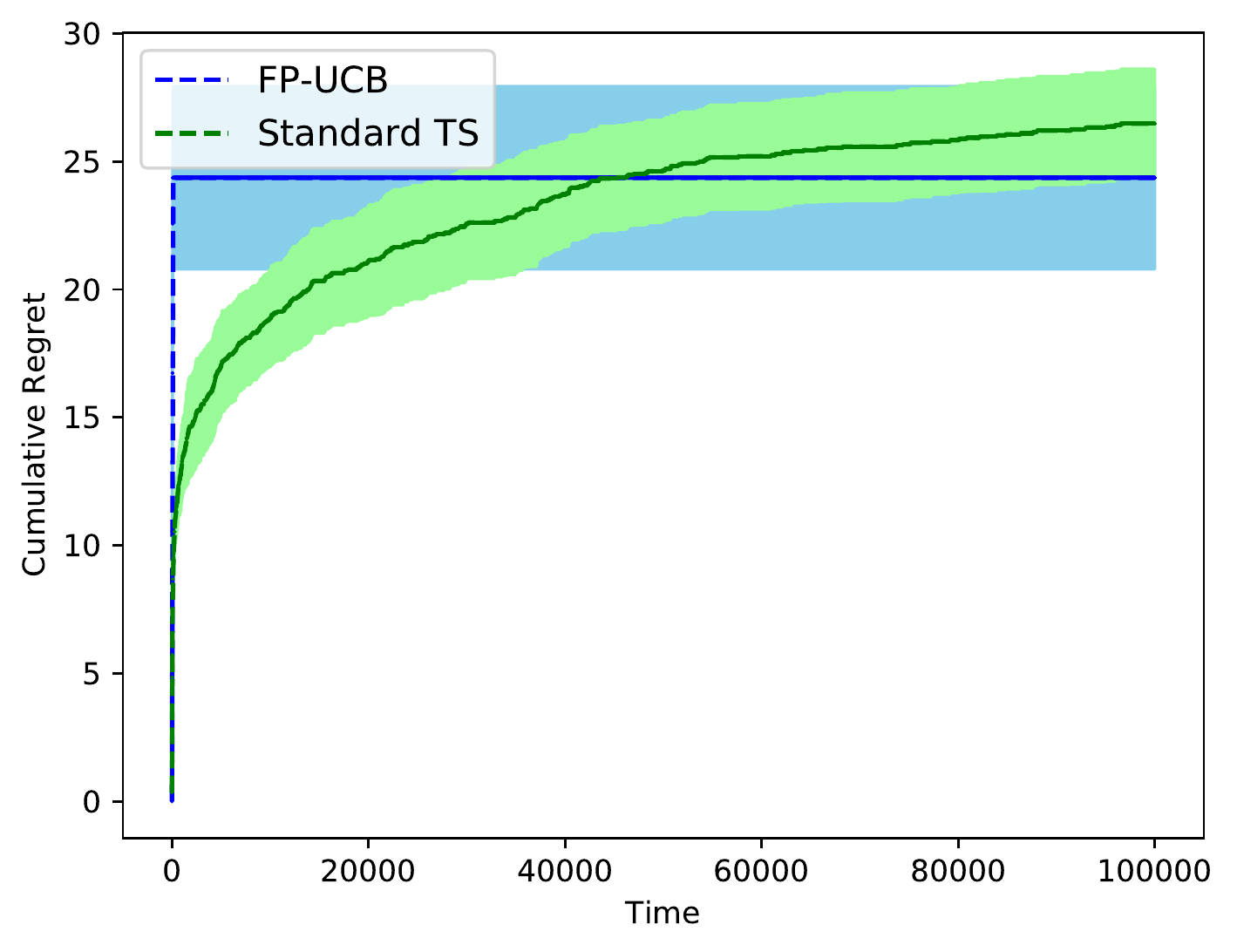}
		\captionof{figure}{}
		\label{fig:comparision-2}
	\end{minipage}
%	\begin{minipage}{.3\textwidth}
%		\centering
%		\includegraphics[width=\linewidth]{../Simulations/fpucb_ts_ucb_1_pd.pdf}
%		\captionof{figure}{}
%		\label{fig:comparision-12}
%	\end{minipage}
	\begin{minipage}{.3\textwidth}
	\centering
	\includegraphics[width=\linewidth]{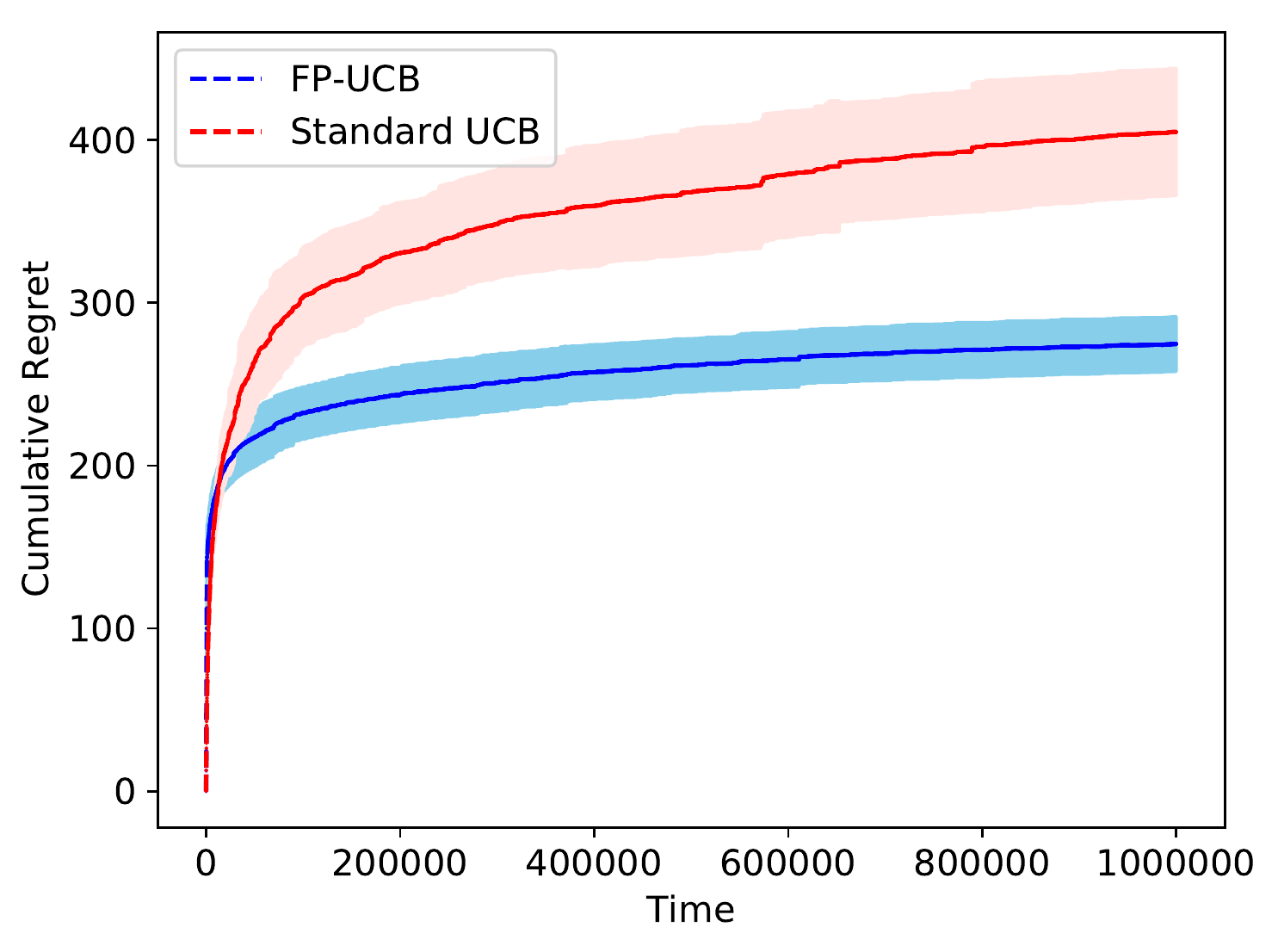}
	\captionof{figure}{}
	\label{fig:comparision-3}
\end{minipage}
	
\end{figure*}

\section{Conclusion and  Future Work} 

We proposed an algorithm for finitely parameterized multi-armed bandits. Our FP-UCB algorithm achieves bounded regret if the parameter set satisfies some necessary condition and logarithmic regret in other cases. In both cases, the theoretical performance guarantees for our algorithm are superior to the standard UCB algorithm for multi-armed bandits. Our algorithm also shows superior numerical performance. 

%We came up with FP-UCB which uses confidence sets to balance exploration and exploitation in MAB problems with a structure suitable for many real-world applications. We developed an achievable generic regret bound for FP-UCB with improved constants that capture the structure of the problem. In practice, the FP-UCB algorithm is easy to implement to get a desired result.

In the future, we will extend this approach to linear bandits and contextual bandits. Reinforcement learning problems where the underlying MDP is finitely parameterized is another research direction we plan to explore.  We will also develop similar algorithms using Thompson sampling approaches. 

%We will also  
%
%A future direction would be to handle large scale reinforcement learning problems with Markovian dynamics. It would be promising if FP-UCB scales up to solve a computationally demanding large-state space Markov Decision Processes; this has been previously attempted with a promising regret bounds \cite{gopalan2015TSmdp}. We can also attempt to develop a theoretical understanding for problems with a continuous state space. We also can analyze Bayesian regret for this problem setting leveraging Thompson samping and thus providing a rigorous characterization of the same.

%\input{biblio}

\bibliographystyle{ieeetr}
\bibliography{MAB-References}

\section*{Appendix}
%\section{Analysis of FP-UCB Algorithm}
%%\label{section:ucb-analysis}

\subsection{Proof of Lemma \ref{lemma_beta_delta}}

%\begin{lemma}\label{lemma_beta_delta}
%Let $\Delta_{i}$ and $\beta_{i}$ be as defined in \eqref{eq:delta_i} and \eqref{eq:beta_i} respectively. 	Then, for each $i \in C(\theta^{o})$, $\beta_{i} > 0$. Moreover,  $\beta_{i} > \Delta_{i}$.
%\end{lemma}

\begin{proof}
	Fix an $i \in C(\theta^{o})$. Then there exists a $\theta \in B(\theta^{o})$ such that $a^*(\theta)=i$. For this $\theta$, by the definition of $B(\theta^{o})$, we have
	\begin{align}
	\label{eq:lemma_1}
	\mu_1(\theta^o) = \mu_1(\theta).
	\end{align}
	Using Assumption \ref{assumption:unique}, it follows that
	\[\hspace{-0.25cm} \mu_{i}(\theta) = \mu_{a^{*}(\theta)}(\theta) > \mu_1(\theta) = \mu_1(\theta^o) =   \mu_{a^{*}(\theta^{o})}(\theta^{o})   > \mu_{i}(\theta^o). \] 
	Thus, $\beta_{i}=\min_{\theta:\theta\in B(\theta^o),a^*(\theta)=i} |\mu_{i}(\theta^o) - \mu_{i}(\theta) | > 0.$

	Now, for any given $\theta$  considered above, suppose $|\mu_{i}(\theta) - \mu_{i}(\theta^o) |\leq \Delta_{i}$. Since $\Delta_{i} > 0$ by definition, this implies that
	\begin{align*}
	\mu_{a^{*}(\theta)}(\theta) =  \mu_{i}(\theta)  \leq \Delta_i + \mu_{i}(\theta^o) ~\substack{(a)\\=}~ \mu_1(\theta^o) - \mu_i(\theta^o) + \mu_{i}(\theta^o) = \mu_1(\theta^{o})~ \substack{(b)\\=}~  \mu_1(\theta),
	\end{align*} 
	where (a) follows from definition of $\Delta_i$ and (b) from \eqref{eq:lemma_1}. This is a contradiction because  $\mu_{a^{*}(\theta)}(\theta)>\mu_1(\theta)$. 
	
	Thus,  $|\mu_{i}(\theta) - \mu_{i}(\theta^o) | > \Delta_{i}$  for any $\theta\in B(\theta^o)$ such that $a^*(\theta)=i$. So, $\beta_{i} > \Delta_{i}.$
\end{proof}

\subsection{Proof of Lemma \ref{lemma_c_i}}

\begin{proof}
	We have  $C_i= 1+4|A|+\min_{\theta:a^*(\theta)=i} ( k_{i}(\theta) + K_i(\theta))$.
	
	First recall that $k_{i}(\theta) := \min \left\{k:  k \geq 3,  k>\lceil 12\log (k)/\alpha^{2}_1(\theta) \rceil  \right\}$. Since $\log(x) \leq (x-1)/\sqrt{x}$ for all $1 \leq x < \infty$, we have \[\left\{k:  k \geq 3,  k> \frac{12(k-1)}{\alpha^{2}_1(\theta) \sqrt{k}} + 1 \right\} \subseteq \left\{k:  k \geq 3,  k>\lceil 12\log (k)/\alpha^{2}_1(\theta) \rceil  \right\}.\] The Left-Hand-Side of the above equation simplifies to $\left\{k:  k \geq 3,  k>144/\alpha^{4}_1(\theta)  \right\} $. Thus, we have $k_{i}(\theta) \leq \max\{ 3, \lceil 144/\alpha^{4}_1(\theta) \rceil  \}.$
	
	Now, recall that   $K_i(\theta)$ is defined as 
	\begin{align}
	K_i(\theta) &= \sum_{k=k_{i}(\theta)}^T  \frac{2|A| k}{(   k-\left\lceil \frac{12\log (k)}{\alpha^{2}_{1}(\theta)} \right\rceil  )^{5}} \nonumber\\
	&\leq \sum_{k=k_{i}(\theta)}^\infty  \frac{2|A| k}{(   k-\left\lceil \frac{12\log (k)}{\alpha^{2}_{1}(\theta)} \right\rceil  )^{5}} \nonumber \\
	&= \sum_{k=k_{i}(\theta)}^{E_i(\theta) }  \frac{2|A| k}{(   k-\left\lceil \frac{12\log (k)}{\alpha^{2}_{1}(\theta)} \right\rceil  )^{5}} +  \sum_{k=E_i(\theta) +1}^{\infty}  \frac{2|A| k}{(   k-\left\lceil \frac{12\log (k)}{\alpha^{2}_{1}(\theta)} \right\rceil  )^{5}}. \label{eq:split}
	\end{align} %\emph{Aside:} We note that there was a typo in the definition of $K_i(\theta)$ in the manuscript at equation (15) in (f). The typo was that we had ``$\leq$'' in (f) instead of an equality, this change is colored in \bluetext{blue}.
	
	We analyze the first summation in \eqref{eq:split}. Thus, we get, \begin{align}
	\sum_{k=k_{i}(\theta)}^{E_i(\theta) }  \frac{2|A| k}{(   k-\left\lceil \frac{12\log (k)}{\alpha^{2}_{1}(\theta)} \right\rceil  )^{5}} \leq \sum_{k=k_{i}(\theta)}^{E_i(\theta) }  2|A| k &\leq \sum_{k=1}^{E_i(\theta) }  2|A| k = E_i(\theta) (E_i(\theta)+1)   |A|. \label{eq:p2}
	\end{align}
	
	Since $\log(x) \leq (x-1)/\sqrt{x}$ for all $1 \leq x < \infty$, we have \[  k-\left\lceil \frac{12\log (k)}{\alpha^{2}_{1}(\theta)} \right\rceil  \geq  k- \frac{12\log (k)}{\alpha^{2}_{1}(\theta)}   - 1 \geq \frac{(k-1)(\alpha^{2}_{1}(\theta)\sqrt{k}-12)}{\alpha^{2}_{1}(\theta) \sqrt{k}}. \] Using this, the second summation in \eqref{eq:split} can be bounded as \begin{align} 
	\sum_{k=E_i(\theta) +1}^{\infty}  \frac{2|A| k}{(   k-\left\lceil \frac{12\log (k)}{\alpha^{2}_{1}(\theta)} \right\rceil  )^{5}}  &\leq  \sum_{k=E_i(\theta) +1}^{\infty}  \frac{2|A| k^{7/2} \alpha^{10}_{1}(\theta) }{( (k-1)(\alpha^{2}_{1}(\theta)\sqrt{k}-12) )^{5}} \nonumber \\
	&\stackrel{(a)}{\leq}  \sum_{k=E_i(\theta) +1}^{\infty}  \frac{2|A| k^{7/2} \alpha^{10}_{1}(\theta) }{( k-1)^{5}} \nonumber \\
	&\leq 2|A|  \alpha^{10}_{1}(\theta)  \sum_{k=4}^{\infty}  \frac{k^{7/2}  }{( k-1)^{5}} \stackrel{(b)}{\leq} 4|A| \alpha^{10}_{1}(\theta)   \label{eq:p3}
	\end{align} where (a) follows from the observation that $(\alpha^{2}_{1}(\theta)\sqrt{k}-12)>1$ for $k\geq E_i(\theta) +1$ and (b) follows from calculus (an integral bound).
	
	Thus using \eqref{eq:p2} and \eqref{eq:p3} in \eqref{eq:split}, we get $K_i(\theta) \leq E_i(\theta) (E_i(\theta)+1)   |A| + 4|A| \alpha^{10}_{1}(\theta).$ This concludes the proof of this lemma.
\end{proof}

\end{document}